\documentclass{article}
\usepackage{amsthm}

\usepackage[T1]{fontenc}       
\usepackage{microtype}         

\usepackage{graphicx}          
\usepackage{wrapfig}           
\usepackage{float}             
\usepackage{adjustbox}         
\usepackage{caption}           
\usepackage{subcaption}        

\RequirePackage[table,x11names]{xcolor}
\definecolor{richblack}{rgb}{0.0, 0.25, 0.25}
\definecolor{darkcerulean}{rgb}{0.03, 0.27, 0.49}
\definecolor{smokyblack}{rgb}{0.06, 0.05, 0.03}
\definecolor{warmblack}{rgb}{0.0, 0.26, 0.26}
\definecolor{cobalt}{rgb}{0.0, 0.28, 0.67}
\definecolor{darkcobalt}{rgb}{0.1, 0.38, 0.77}

\usepackage[most]{tcolorbox}  
\tcbuselibrary{breakable}

\tcbset{
  userstyle/.style={
    colback=white, colframe=black, fontupper=\color{black},
    left=1mm, right=1mm, top=1mm, bottom=1mm,
    boxrule=0.5pt, arc=2mm,
  },
  origprompt/.style={
    colback=blue!10, colframe=blue!40!black, fontupper=\color{black},
  },
  rephrasestyle/.style={
    colback=orange!10, colframe=orange!60!black, fontupper=\color{black},
  },
  botstyle/.style={
    colback=gray!5, colframe=gray!40!black, fontupper=\color{black},
  },
}

\usepackage{mathtools}         
\usepackage{amssymb}           
\usepackage{amsfonts}          
\usepackage{dsfont}            

\usepackage{natbib}
\usepackage[colorlinks=true,
            linkcolor=darkcobalt,
            citecolor=cobalt,
            urlcolor=darkcerulean]{hyperref}

\usepackage{makecell}          
\usepackage{booktabs}          
\usepackage{multirow}          
\usepackage{lscape}            

\usepackage{enumitem}          
\usepackage{fancyvrb}          

\usepackage{comment}           
\usepackage{silence}           
\usepackage{afterpage}         
\usepackage[overload]{empheq}  
\usepackage{appendix}
\usepackage{titletoc}          
\usepackage{cleveref}          
\usepackage{subcaption}

\usepackage[UKenglish]{isodate}
\usepackage[format=default,font=small,labelfont=it]{caption}
\usepackage{afterpage}
\usepackage{rotating}

\definecolor{ForestGreen}{rgb}{0.0, 0.27, 0.13}

\makeatletter
\newcommand{\pushright}[1]{\ifmeasuring@#1\else\omit\hfill$\displaystyle#1$\fi\ignorespaces}
\newcommand{\pushleft}[1]{\ifmeasuring@#1\else\omit$\displaystyle#1$\hfill\fi\ignorespaces}
\makeatother

\renewcommand{\phi}{\varphi}

\newcommand{\eqdef}{\ensuremath{\,\raisebox{-1.2pt}{${\stackrel{\mbox{\upshape \scalebox{.42}{def.}}}{=}}$}}\,}

\newtheorem{proposition}{Proposition}
\newtheorem{assumption}{Assumption}
\newtheorem{theorem}{Theorem}

\NewDocumentCommand{\luca}{mo}{
    \IfValueF{#2}{
                        {{\scriptsize
                            \textcolor{green}{ 
                            \textbf{L:}
                            \textit{{#1}}
                            }
                        }}
        }
    \IfValueT{#2}{
                        \marginnote{{\scriptsize
                            \textcolor{green}{ 
                            \textbf{L:}
                            \textit{{#1}}
                            }
                        }}
        }
                    }
\NewDocumentCommand{\giulia}{mo}{
    \IfValueF{#2}{
                        {{\scriptsize
                            \textcolor{red}{ 
                            \textbf{GL:}
                            \textit{{#1}}
                            }
                        }}
        }
    \IfValueT{#2}{
                        \marginnote{{\scriptsize
                            \textcolor{red}{ 
                            \textbf{GL:}
                            \textit{{#1}}
                            }
                        }}
        }
}
\NewDocumentCommand{\anastasis}{mo}{
    \IfValueF{#2}{
                        {{\scriptsize
                            \textcolor{violet}{ 
                            \textbf{A:}
                            \textit{{#1}}
                            }
                        }}
        }
    \IfValueT{#2}{
                        \marginnote{{\scriptsize
                            \textcolor{violet}{ 
                            \textbf{A:}
                            \textit{{#1}}
                            }
                        }}
        }
                    }
                    
\NewDocumentCommand{\cody}{mo}{
    \IfValueF{#2}{
                        {{\scriptsize
                            \textcolor{orange}{ 
                            \textbf{A:}
                            \textit{{#1}}
                            }
                        }}
        }
    \IfValueT{#2}{
                        \marginnote{{\scriptsize
                            \textcolor{orange}{ 
                            \textbf{A:}
                            \textit{{#1}}
                            }
                        }}
        }
                    }

\NewDocumentCommand{\yannick}{mo}{
    \IfValueF{#2}{
                        {{\scriptsize
                            \textcolor{cyan}{ 
                            \textbf{Y:}
                            \textit{{#1}}
                            }
                        }}
        }
    \IfValueT{#2}{
                        \marginnote{{\scriptsize
                            \textcolor{cyan}{ 
                            \textbf{Y:}
                            \textit{{#1}}
                            }
                        }}
        }
                    } 

\definecolor{darkgreen}{rgb}{0.0, 0.2, 0.13}
\NewDocumentCommand{\xuwei}{mo}{
    \IfValueF{#2}{
                        {{\scriptsize
                            \textcolor{darkgreen}{ 
                            \textbf{X:}
                            \textit{{#1}}
                            }
                        }}
        }
    \IfValueT{#2}{
                        \marginnote{{\scriptsize
                            \textcolor{darkgreen}{ 
                            \textbf{X:}
                            \textit{{#1}}
                            }
                        }}
        }
                    }

\newcounter{termcounter}
\renewcommand{\thetermcounter}{\Roman{termcounter}}
\crefname{term}{term}{terms}
\creflabelformat{term}{#2\textup{(#1)}#3}

\makeatletter
\def\term{\@ifnextchar[\term@optarg\term@noarg}
\def\term@optarg[#1]#2{%
  \textup{#1}%
  \def\@currentlabel{#1}%
  \def\cref@currentlabel{[][2147483647][]#1}%
  \cref@label[term]{#2}}
\def\term@noarg#1{%
  \refstepcounter{termcounter}%
  \textup{(\thetermcounter)}%
  \cref@label[term]{#1}}
\makeatother

\crefname{lemma}{lemma}{lemmata}
\Crefname{lemma}{Lemma}{Lemmata}
\crefname{notation}{notation}{notations}
\Crefname{notation}{Notation}{Notations}
\crefname{assumption}{assumption}{assumptions}
\crefname{example}{Example}{Examples}
\crefname{proposition}{Proposition}{Proposition}

\usepackage[accepted]{icml2025}

\icmltitlerunning{LoRA Fine-Tuning Without GPUs: A CPU-Efficient Meta-Generation Framework for LLMs}

\begin{document}

\twocolumn[
\icmltitle{LoRA Fine-Tuning Without GPUs: \\ A CPU-Efficient Meta-Generation Framework for LLMs}



\icmlsetsymbol{equal}{*}

\begin{icmlauthorlist}
\icmlauthor{Reza Arabpour}{McMaster,Vector}
\icmlauthor{Haitz Sáez de Ocáriz Borde}{Oxford}
\icmlauthor{Anastasis Kratsios}{McMaster,Vector}
\end{icmlauthorlist}

\icmlaffiliation{McMaster}{Department of Mathematics, McMaster University, Hamilton, Canada}
\icmlaffiliation{Vector}{Vector Institute, Toronto, Canada}
\icmlaffiliation{Oxford}{University of Oxford, Oxford, United Kingdom}

\icmlcorrespondingauthor{Reza Arabpour}{arabpour@mcmaster.ca}
\icmlcorrespondingauthor{Anastasis Kratsios}{kratsioa@mcmaster.ca}
\icmlcorrespondingauthor{Haitz Sáez de Ocáriz Borde}{chri6704@ox.ac.uk}

\icmlkeywords{Machine Learning, ICML}

\vskip 0.3in
]



\printAffiliationsAndNotice{\icmlEqualContribution} 

\begin{abstract}
Low-Rank Adapters (LoRAs) have transformed the fine-tuning of Large Language Models (LLMs) by enabling parameter-efficient updates. However, their widespread adoption remains limited by the reliance on GPU-based training. In this work, we propose a theoretically grounded approach to LoRA fine-tuning designed specifically for users with limited computational resources, particularly those restricted to standard laptop CPUs. Our method learns a meta-operator that maps any input dataset, represented as a probability distribution, to a set of LoRA weights by leveraging a large bank of pre-trained adapters for the Mistral-7B-Instruct-v0.2 model. Instead of performing new gradient-based updates, our pipeline constructs adapters via lightweight combinations of existing LoRAs directly on CPU. While the resulting adapters do not match the performance of GPU-trained counterparts, they consistently outperform the base Mistral model on downstream tasks, offering a practical and accessible alternative to traditional GPU-based fine-tuning.
\end{abstract}

\section{Introduction}

As models and datasets scale up, full fine-tuning becomes increasingly unrealistic for most practitioners. The largest foundation models—often built by tech giants with almost unlimited compute~\citep{touvron2023,openai2023,bai2023qwentechnicalreport,qwen2025qwen25technicalreport,deepseekai2025deepseekv3technicalreport}—can have hundreds of billions of parameters, making traditional fine-tuning for individuals prohibitively expensive. Parameter-efficient fine-tuning (PEFT) methods~\citep{he2022towards, pfeiffer-etal-2020-adapterhub, ding2022deltatuningcomprehensivestudy, yu2022vpetl, han2024parameterefficientfinetuninglargemodels} offer a workaround: instead of updating all weights, they tweak a small subset, slashing compute and storage costs while maintaining reasonable performance. Among these, the Low-Rank Adapter (LoRA)~\citep{hu2021loralowrankadaptationlarge} approach has become standard due to combined simplicity and surprisingly powerful effectiveness. Nevertheless, for modern massive LLMs, LoRA fine-tuning can still be long and heavy. Thus, the following question becomes necessary:

\vspace{-0.5em}

\begin{center}
\textit{Can one generate new low-rank adapters to fine-tune large language models on new tasks without the need for GPUs?}  
\end{center}

\vspace{-0.5em}

We address this concern by introducing a \textit{zero-shot LoRA meta-generation procedure aimed at CPU-only users}. Our approach takes novel datasets, each potentially containing a variable number of instances, as input. It then outputs LoRA weights for a pre-trained LLM, where the prediction relies on a combination of instances from an existing bank of LoRAs~\citep{gabrielsson2024compress}. Importantly, the way in which these combinations are performed is lightweight enough to be computable on a standard contemporary CPU in a few minutes (see Table~\ref{table:timing} in Appendix~\ref{s:Implementation_Details}), with no need for GPU clusters.

\paragraph{Main Contribution}
Our principled LoRA meta-generation pipeline provides lightweight, ``cheap'' LoRAs that approach the performance of GPU-fine-tuned models (which are often inaccessible to many) and outperform the base ``non-finetuned'' model. These contributions are theoretically grounded in Proposition~\ref{prop:PAC} and Theorem~\ref{thrm:Main}. Together, these demonstrate that, with high probability, a ReLU Multi-Layer Perceptron (MLP) architecture, designed to run efficiently on a CPU, can identify the optimal coefficients for combining existing LoRAs. These optimal LoRA mixture coefficients, as defined in~\eqref{eq:predicted_mixture}  (representing a weighted sum of pre-trained LoRA parameters), are determined based on the given dataset alignment features. This process effectively minimizes the downstream task loss, which quantifies the model's error on new, specific tasks. Additionally, our work also provides nearly-optimal closed-form solutions through lightweight, neural network-free alternatives (e.g., the Attentional or Normalized approaches). Interestingly, our experiments reveal that the neural network-free variants of our pipeline perform comparably to the theoretically near-optimal neural network solution (the MLP-based approach).  

Section~\ref{sec:Related Work} provides a discussion of related work concerning LoRA. We introduce the preliminaries for formalizing datasets as probability distributions in Section~\ref{sec:Preliminaries}. Section~\ref{s:Pipelines} presents our LoRA generation pipelines. Their respective theoretical guarantees are later detailed in Section~\ref{s:Main_Result}, and experimentally validated in Section~\ref{s:Experiments}.

\section{Related Work}
\label{sec:Related Work}

Since its introduction, the utility of LoRA~\citep{hu2021loralowrankadaptationlarge} has expanded significantly beyond classical LLM post-training and language. It is now employed in diverse fields such as vision language models~\citep{li2023graphadapter} and vision Transformers~\citep{dong2023efficient}. LoRA has also proven valuable in image generative modeling for rapid Stable Diffusion fine-tuning and personalization~\citep{rombach2022highresolutionimagesynthesislatent,gal2022image,ruiz2022dreambooth,roich2022pivotal}, and for score distillation~\citep{wang2023prolificdreamer}, although more principled LoRA-free methods have recently emerged~\citep{lukoianov2024score}. Its application even extends to fine-tuning base models into reasoning models using reinforcement learning~\citep{wang2025tinatinyreasoningmodels}, and in the development of new adapters for Graph Neural Networks and Graph Transformers~\citep{papageorgiou2025graph}.

Alongside this expanding applicability, numerous LoRA variants have emerged, often aiming to further reduce computational overhead. For instance, quantization offers a way to lower memory consumption both during training~\citep{gholami2021survey_quantization,Dettmers2023QLoRAEF,guo2024lqLoRA} and after~\citep{yadav2023compeft}. The number of trainable parameters can also be reduced through adaptive rank allocation~\citep{zhang2023adaLoRA}. Further inspired by ideas around weight or projection reuse~\citep{frankle2018lottery,Ramanujan_2020_random_iccv}, strategies to decrease trainable LoRA parameters include learning diagonal rescaling of frozen random $B$ and $A$ matrices (VeRA)~\citep{kopiczko2024vera}, deriving $B$ and $A$ from the SVD of the pre-trained $W_0$ and optimizing a smaller matrix in the resulting space (SVDiff)~\citep{han2023svdiff}, learning a linear combination of fixed random matrices (NOLA)~\citep{koohpayegani2023nola}, and fine-tuning with orthogonal matrices (BOFT)~\citep{liu2024boft}. LoRAs have also been explored from a more theoretical viewpoint~\cite{zeng2024the,zhu2024asymmetry,kratsios2025sharpgeneralizationboundsfoundation}.

Note that our focus here is on LoRA generation on CPU, which none of the aforementioned works explore. We would like to reiterate that all our pipelines, \textit{including those using artificial neural networks} can be trained solely using CPUs.

\section{Preliminaries}
\label{sec:Preliminaries}

\paragraph{Datasets as Probability Distributions}
To describe our pipeline, we first need a unified framework for datasets with a varying number of instances. As such, we fix dimensions $d,D\in \mathbb{N}_+$.  
Given our training datasets $D_1,\dots,D_N\subset \mathcal{X}$ for some (non-empty) compact input domain $\mathcal{X}\subseteq \mathbb{R}^{d+D}$ corresponding to one of $N$ possible down-stream tasks $\mathcal{T}_1,\dots,\mathcal{T}_N$ which our Transformer model (Mistral-7B-Instruct-v0.2) $f_{\theta}:\mathbb{R}^d\to \mathbb{R}^D$, whose parameters $\theta\in \mathbb{R}^p$ lie in a $p\gg0$ dimensional Euclidean parameter space.  Since the entries of each dataset are permutation-invariant, then, following the synthetic data generation literature, e.g.~\cite{zamanlooy2024locally}, 
it is natural to represent each dataset $D_n$ as an empirical distribution (probability measure) via

\begin{equation}
\label{eq:conversion_Data2Measures}
    P_{D_n}
=
    \frac1{N_m}\,\sum_{(x,y)\in D_m} \delta_{(x,y)}
\end{equation}
on the domain $\mathcal{X}$ where $N_m\eqdef \#D_n$; i.e.\ $P_{D_n}=
    \sum_{m=1}^{N_m}\,w_m \delta_{(x_m,y_m)}$ with $w_m=1/{N_m}$ for each $m=1,\dots,N_m$.
The support of the measure $P_{D_n}$, namely, $\{(x_1,y_1),\dots, (x_m,y_m)\}$ represent instances in $D_n$ and the weights $w_m\in [0,1]$ sum to $1$, i.e.\ $w$ belongs to the $N_m$ simplex $\Delta_{N_m}\eqdef \{u\in [0,1]^{N_m}:\, \sum_{i=1}^{N_m}\,u_i=1\}$, and represent the relative frequency of instance of data-point in $D_m$.   We denote the set of probability measures on $\mathcal{X}$ by $\mathcal{P}(\mathcal{X})$.


\vspace{-0.5em}

\paragraph{Pipeline Inputs and  Distributional Alignment Scores}
We then choose a \textit{data-similarity score} 
where $
\rho
:\mathcal{P}(\mathcal{X})\times \mathcal{P}(\mathcal{X})\to [0,\infty]$. For this, we choose a (dis)similarity metric between probability distributions (measures) on $\mathcal{X}$, e.g.\ an information-theoretic divergence such as Kullback Leibler (KL) divergence or a metric such as the $1$-Wasserstein distance $\mathcal{W}_1$.  This dissimilarity score then allows us to extract \textit{alignment scores} between any new dataset $D$ (encoded as a probability measure $P_D$ on the data-domain $\mathcal{X}$) and every dataset $(D_n)_{n=1}^N$ in our database via $\operatorname{align}: \mathcal{P}(\mathcal{X})  \rightarrow \Delta_N$
\vspace{-1em}
\begin{equation}
\label{eq:aligment_score}
\begin{aligned}
    \operatorname{align}(P_D)
    \eqdef
    \operatorname{Softmax}\big(
        (
        \rho
        (P_D,P_{D_n})_{n=1}^N
    \big)
.
\end{aligned}
\end{equation}
Once the (softmax-normalized) alignment scores are computed, they are passed to a network, here in our ``proof of concept'' we use a simple MLP (trained on CPU), which yields a set of \textit{mixture weights} $W_D\in\Delta_N$.  
These mixture weights are then used to combine the pre-trained LoRA weights $\theta_1,\dots,\theta_N$, from out database. Note that each LoRA weight $\theta_n$ was specialized for task $\mathcal{T}_n$ and pre-trained on dataset $D_n$.  The output of our model is thus simply the mixture of LoRAs
\vspace{-1em}
\begin{equation}
\label{eq:predicted_mixture}
D\mapsto P_D\mapsto \sum_{n=1}^N\, W_D\, \theta_n
\end{equation}
and lies in the \textit{convex hull} of the pre-trained LoRA weights $\theta_1,\dots,\theta_N$ in the parameter space $\mathbb{R}^p$. Therefore, we only need to learn (or compute, as we will see in Section~\ref{s:Pipelines}) the mapping in~\eqref{eq:predicted_mixture}. Based on this we are able to obtain LoRA weights with no fine-tuning, directly \textit{out-of-the-box}.

\section{LoRA Generation Pipelines for CPU}
\label{s:Pipelines}

We now mathematically formalize our end-to-end cheap LoRA pipelines.  Further details on how these were practically implemented can be found in Appendix~\ref{s:Implementation_Details__ss:Further_MethDetails}. Our main theoretical guarantee (Theorem~\ref{thrm:Main}) is general enough to apply not only to LoRAs fed into transformers but also to nearly any mixture-of-expert-based parameter prediction pipeline. 

\subsection{Setup}

Let $d,D\in\mathbb{N}_+$.
Let $\ell:\mathbb{R}^D\times \mathbb{R}^D\to [0,\infty)$ be Lipschitz. Let $f:\mathbb{R}^p\times \mathbb{R}^d\to \mathbb{R}^D$ be a locally-Lipschitz model, mapping the parameters $\theta \in \mathbb{R}^p$ and an input $x\in \mathbb{R}^d$ to an output $f_{\theta}(x)\in \mathbb{R}^D$.  Also, We are given a pre-trained model $\theta_0\in \mathbb{R}^p$.
%
Purely for simplicity, we consider the standardized data-domain $\mathcal{X}=[0,1]^{d+D}$. Following~\citep{rothfuss2023scalable}.  We henceforth fix a \textit{task distribution} $\mathbb{P}\in \mathcal{P}(\mathcal{S})$ quantifying the probability of selecting any one dataset in $\mathcal{S}$ at random. We consider a metric space of datasets $\mathcal{D}\subseteq \mathcal{P}([0,1]^{d+D})$ metrized by $\rho$, where the topology generated by $\rho$ is no coarser than the topology of convergence in distribution.  We fix a $K\in \mathbb{N}_+$ datasets paired with ``fine-tuned'' model parameters $(D_1,\Delta \theta_1),\dots (D_K,\Delta \theta_K)$ in $\mathcal{D}\times \mathbb{R}^p$.  Let 

\vspace{-0.5em}
\begin{small}
\[
        \operatorname{co}(\Delta \theta)
    \eqdef 
    \{
        \vartheta\in \mathbb{R}^p:\, (\exists w\in \Delta_K)\, \vartheta = \sum_{k=1}^K\, w_k \Delta \theta_K
    \}
\]
\end{small}

where $\Delta_K\eqdef \{w\in [0,1]^K:\, \sum_{k=1}^K\, w_k=1\}$.

\vspace{-0.5em}

\subsection{Very-Cheap LoRAs: Attentional Approach}
Consider the following approach which maps any new incoming dataset $D$ to the following mixture of LoRAs
\begin{equation}
\label{eq:VeryCheep}
    \mathcal{C}_{\text{Att}}(D)
    \eqdef 
    \underbrace{
            [\operatorname{softmin}\circ \operatorname{align}(D)]^{\top}
        }_{\text{LoRA Alignment Scores}}
        \underbrace{
            (\Delta \theta_1,\dots,\Delta \theta_K)
        }_{\text{Pre-Trained LoRAs}}
\end{equation}

We refer to the pipeline in~\eqref{eq:VeryCheep} as our \textit{attentional approach} since the dataset $D_1,\dots,D_K$ play a similar role to the \textit{keys} in attention mechanisms~\cite{vaswani2017attention}. The LoRA alignment scores in~\eqref{eq:VeryCheep} are analogous to contextual \textit{alignment scores}, and the pre-trained LoRA parameters play a similar role to the \textit{value matrices} in~\cite{vaswani2017attention}.   The softmin is used instead of a softmax since maximal distance alignment happens when two datasets have a distance of $0$ from one another, not some arbitrarily large number.
We examine a \textit{normalized} version of distance vector~\eqref{eq:VeryCheep} in our experiments, see Appendix~\ref{sss:Normalized} for details.

\vspace{-0.5em}

\subsection{Cheap Nearly-Optimal LoRAs: Neural Approach}
Our neural approach injects non-linear flexibility into how the distances are mapped to the LoRA alignment scores in~\eqref{eq:VeryCheep} using a deep learning model $\mathcal{C}: \mathcal{D} \rightarrow \operatorname{co}(\Delta \theta)$; in this paper, this will always be an MLP.  This allows our cheap LoRA approach to learn how to detect and align complicated non-linear alignments between the new dataset and those defining each pre-trained task.  This \textit{neural} approach thus sends any dataset $D$ to the following mixture of LoRAs
\begin{equation}
\label{eq:Neural}
\begin{aligned}
        \mathcal{C}(
            D
        )
    \eqdef 
        \underbrace{
            [\operatorname{softmin}\circ \hat{f}\circ \operatorname{align}(D)]^{\top}
        }_{\text{Neural-LoRA Alignment Scores}}
        \underbrace{
            (\Delta \theta_1,\dots,\Delta \theta_K)
        }_{\text{Pre-Trained LoRAs}}
\end{aligned}
\end{equation}
where $\hat{f}:\mathbb{R}^K\to \mathbb{R}^K$ is an MLP with activation function $\varsigma$, and we write $\operatorname{align}(D)$ in place of $\operatorname{align}(P_D)$ understanding the correspondence $D\rightarrow P_D$ as implicit.

\section{Theoretical Guarantees}
\label{s:Main_Result}

We now provide guarantees on the optimality of both our main approaches.  We also demonstrate the existence of an oracle optimizer, yielding the best possible LoRA if the user had access to complete information on the task distribution.

\vspace{-0.5em}

\subsection{Attentional Approach}

Our cheapest out-of-the-box LoRA pipeline~\eqref{eq:VeryCheep} are optimal in a PAC-Bayesian sense of~\cite{alquier2016properties}.
\begin{proposition}[\small{Existence: Optimal Oracles for Fine-Tuning}]
\label{prop:PAC}
For every $K\in \mathbb{N}_+$ and $\{(D_k,\Delta \theta_k)\}_{k=1}^K \subset \mathcal{D}\times \mathbb{R}^p$ with each $D_k$ finite and non-empty.  For every  $\alpha>0$ and each dataset $D\in \mathcal{D}$, the LoRA Alignment Scores in~\eqref{eq:VeryCheep} satisfy
\vspace{-0.5em}
\begin{equation*}
\resizebox{0.91\hsize}{!}{$
        \underbrace{
            \operatorname{softmin}\circ \operatorname{align}(D)
        }_{\text{LoRA Alignment Scores}}
    \in
        \operatorname{argmin}_{w\in \Delta_k}
        \underbrace{
            \frac1{K}\sum_{k=1}^K\,
                w_k\, \rho(D,D_k)
        }_{\text{Dataset Alignment}}
+
    \underbrace{
        \frac1{\alpha}\,
        \sum_{k=1}^K\, w_k\log(w_k)
    }_{\text{Entropic Penalty}}
$}
\end{equation*}
\end{proposition}

\begin{proof}
See Appendix~\ref{s:Proofs__ss:PAC}.
\end{proof}

\vspace{-1em}

\subsection{Neural Approach}
The attentional pipeline, in~\eqref{eq:VeryCheep}, only checks for the \textit{alignment} of a dataset with the datasets previously used for training the adapters in the bank. In contrast our neural approach, in~\eqref{eq:Neural}, optimizes for the \textit{downstream performance} of the predicted mixture of LoRA experts.  Surprisingly, at least theoretically, one only needs a small MLP between the distance vector and softmin normalization layers to perform this out-of-the-box downstream (near) optimal LoRA generation. Our first guarantee for the neural approach demonstrates the existence of a map, i.e., an oracle predictor, which returns the best possible downstream optimization.

\begin{proposition}[\small{Existence: Optimal Oracles for Fine-Tuning}]
\label{prop:Existence}
For every dataset $D\in \mathcal{D}$ there exists an \textit{oracle} parameter $\vartheta^{\star}\in \operatorname{co}(\Delta \theta)$ satisfying
\vspace{-0.5em}
\begin{equation*}
\resizebox{0.91\hsize}{!}{$%
\underbrace{
        \mathbb{E}_{(X,Y)\sim \mathcal{D}}\Big[
            \ell(f_{\theta + \vartheta^{\star}}(X),Y)
        \Big]
}_{\text{Oracle Error}} = 
\underbrace{
\inf_{\Delta \theta \in \operatorname{co}(\Delta \theta)}\,
        \mathbb{E}_{(X,Y)\sim \mathcal{D}}\Big[
            \ell(f_{\theta + \Delta \theta}(X),Y)
        \Big].
}_{\text{Optimal Error}}
$}
\end{equation*}
\end{proposition}
\begin{proof}
See Appendix~\ref{s:Proofs__ssTHRMMAIN}.
\end{proof}

\vspace{-0.5em}

Our next and main results show that our pipeline can implement the optimal downstream mixture of LoRA predictors to achieve precision.  Our result only relies on one structural regularity condition on our data, guaranteeing that: the \textit{inverse problem }of recording a dataset/measure from its distance measurements to the available datasets/measures is possible.  Effectively, this means that the metric dimension, in the graph-theoretic sense (see~\cite{tillquist2023getting} for details), of the space $\mathcal{D}$ is exactly $K$.
\begin{assumption}[Well-Posed Inverse Problem]
\label{ass}
Let $(\mathcal{D},\rho)$ be compact and suppose that $\rho$ metrizes the weak topology (convergence in distribution) on $\mathcal{D}$.
We require that: the map $\operatorname{align}:\mathcal{D}\to [0,\infty)^K$ \textit{injectively} maps any $D\in \mathcal{D}$ to 
\[
    \operatorname{align}(D)\eqdef \big(\rho(D,D_k)\big)_{k=1}^K
.
\]
\end{assumption}
\begin{theorem}[$\varepsilon$-Optimal Cheap Fine-tuning]
\label{thrm:Main}
Let $\varsigma:\mathbb{R}\to \mathbb{R}$ be a Lipschitz activation function which is differentiable with non-zero derivative at least on point.
For every $0<\varepsilon\le 1$, there is a MLP $\mathcal{C}:\mathbb{R}^K\to \mathbb{R}^K$ with activation function $\varsigma$ such that the $\epsilon$-optimal selection property:
\[
\resizebox{0.91\hsize}{!}{$%
        \underbrace{
            \mathbb{E}_{(X,Y)\sim \mathcal{D}}\Big[
                    \ell(f_{\theta + \mathcal{C}(D)}(X),Y)
                \Big]
        }_{\text{Cheap Fine-Tuning}}
    \le
        \underbrace{
            \inf_{\Delta \theta \in \operatorname{co}(\Delta \theta)}\,
                \mathbb{E}_{(X,Y)\sim \mathcal{D}}\Big[
                    \ell(f_{\theta + \Delta \theta}(X),Y)
                \Big]
        }_{\text{Fine-Turning Oracle}}
    +
        \varepsilon
$}
\]
holds with $\mathbb{P}$-probability at-least $1-\varepsilon$.
\end{theorem}
\begin{proof}
See Appendix~\ref{s:Proofs__ssTHRMMAIN}.
\end{proof}

\section{Experimental Results}
\label{s:Experiments}
A comprehensive evaluation was conducted to assess the performance of three distinct approaches (Attentional, Normalized, and Neural) in conjunction with four established distance metrics (or divergences): Wasserstein Distance (WD), Kullback–Leibler (KL) divergence, Jensen-Shannon (JS) divergence, and Maximum Mean Discrepancy (MMD). This evaluation aimed to systematically compare the outputs generated by each combination of approach and metric. The primary evaluation criterion for the quality of the generated adapters was Rouge-L, a metric ranging from 0 to 1 that quantifies similarity based on the overlap of the longest common subsequences between generated and reference outputs~\cite{lin-2004-rouge}. We also include Exact Match~(EM) results in Appendix~\ref{appendix: exact match}.

Our experimental setup utilized the Mistral-7B-Instruct-v0.2 model \cite{jiang2023mistral7b} and a dataset comprising 502 English dataset-adapter pairs sourced from the Lots-of-LoRAs HuggingFace repository \cite{gabrielsson2024compress}. Further technical details regarding the implementation are provided in Appendix \ref{s:Implementation_Details}.

Our experimental setup highlighted a key distinction in resource usage: the actual computation and adaptation of the LoRA adapters were performed exclusively on the CPU. GPUs, however, were essential only for the evaluation phase. This is because each adapted LLM, after being modified by our pipeline, needed to be loaded onto a GPU to generate outputs on its respective test set. To thoroughly assess the performance of each approach-distance (or divergence) metric pairing, we executed the entire pipeline twelve times for each of the 502 datasets. This exhaustive process covered every unique combination of approaches and distance metrics. Following the generation of outputs, the Rouge-L score was calculated for the test set of each dataset, and the reported values reflect the average of these scores across all runs.
\vspace{-0.5em}

\subsection{Performance Comparison and Analysis}
Our work is benchmarked against two key performance indicators. First, the performance of the base foundation model without any fine-tuning, representing a scenario where an end-user with limited computational resources applies a foundation model to a new dataset: this yielded an average and standard deviation Rouge-L score of $0.192 \pm 0.181$. Second, we compare against the performance of a GPU-fine-tuned model, achieved without hardware limitations, which obtained an Rouge-L score of $0.746 \pm 0.265$. Table \ref{table:RougeLPerformance} presents the average and standard deviation of Rouge-L performance for all approaches across the four distance (or divergence) metrics on the downstream task.

The JS divergence-based Normalized approach achieved the highest score, with an average Rouge-L of 0.520. This represents an improvement of 0.328 over the base model's score of 0.192. It is worth mentioning that even our Attentional approach, despite its simplicity, significantly outperforms the base foundation model across all distance metrics. Interestingly, the neural approach does not seem to justify the additional computational cost, as its performance improvement over the Attentional and Normalized approaches is generally minimal or even worse.

\begin{table}[h]
\caption{Performance of our cheap LoRA pipelines.}
\centering
\small
\adjustbox{width=1\linewidth}{%
\begin{tabular}{lcccc}
\toprule
Approach    & WD       & KL          & JS         & MMD \\
\midrule
Attentional & $0.426$ & $0.501$ & $0.486$ & $0.486$ \\
(std. dev.) & $(\pm0.290)$ & $(\pm0.272)$ & $(\pm0.270)$ & $(\pm0.270)$ \\
\midrule
Normalized  & $0.495$ & $0.488$ & $\mathbf{0.520}$ & $0.497$ \\
(std. dev.) & $(\pm0.267)$ & $(\pm0.269)$ & $(\pm0.277)$ & $(\pm0.269)$ \\
\midrule
Neural      & $0.494$ & $0.482$ & $0.484$ & $0.493$ \\
(std. dev.) & $(\pm0.265)$ & $(\pm0.268)$ & $(\pm0.272)$ & $(\pm0.270)$ \\
\bottomrule
\end{tabular}
}
\label{table:RougeLPerformance}
\end{table}

\section{Conclusion}
In conclusion, our work presents a practical, simple, and theoretically supported pipeline for generating LoRAs suitable for fine-tuning LLMs using only a CPU. This pipeline significantly reduces the typically required computational demands, making fine-tuning accessible even to users with limited hardware resources or on edge devices with privacy constraints. 

We proved the existence of a lightweight ReLU MLP backbone, runnable on a CPU, that can reliably approximate optimal LoRA adapter weights and biases, thereby effectively minimizing downstream task loss in Theorem \ref{thrm:Main}. Surprisingly, the simplest versions of our pipeline (Attentional and Normalized) achieved performance matching that of the MLP backbone version, further demonstrating the efficiency and power of our approach.

Our experiments, using the Mistral-7B-Instruct-v0.2 model on 502 diverse datasets, demonstrate substantial improvements over the baseline model, with the best configuration achieving a 0.328 increase in performance (Rouge-L score) over the base model, bridging more than half of the performance gap between the base model and the GPU fine-tuned reference. While our CPU-generated adapters do not yet match the performance of GPU-trained adapters, they provide a compelling alternative in resource-limited settings.

Future work could explore the applicability of these approaches to other language models as more LoRA adapter banks become open-source, as well as to tasks beyond NLP. Likewise, it would be of interest to better understand how many LoRA adapters would be required to generate new, high-quality adapters—that is, what size of bank is necessary. We expect this to depend on the task, data modalities, and possibly even the model architecture. Finally, our method could also potentially be used for LoRA initialization (pre-heating) before fine-tuning on GPU.


\bibliographystyle{unsrtnat}
\bibliography{LoRAs,TheoryRefs}

\newpage
\appendix
\onecolumn

\section*{Funding}
A.\ Kratsios and R.\ Arabpour acknowledge financial support from an NSERC Discovery Grant No.\ RGPIN-2023-04482 and No.\ DGECR-2023-00230.  They also acknowledge that resources used in preparing this research were provided, in part, by the Province of Ontario, the Government of Canada through CIFAR, and companies sponsoring the Vector Institute\footnote{\href{https://vectorinstitute.ai/partnerships/current-partners/}{https://vectorinstitute.ai/partnerships/current-partners/}}.

\section{Additional Background}
\label{s:Background}
This appendix presents any additional background required for the formulation of our main results, proofs of our guarantees, and additional experimental details.

\subsection{Foundation Model Fine-tuning and Attention Layers} In modern LLMs, fine-tuning all parameters can be computationally expensive and memory-intensive. LoRA \cite{hu2021loralowrankadaptationlarge} provides an efficient alternative by introducing low-rank updates to pre-trained weight matrices, particularly focusing on attention layers in transformer-based models.  Given query \(Q\), key \(K\), and value \(V\) matrices, the standard attention mechanism computes the attention scores
\begin{equation}
\text{Attention}(Q, K, V) = \text{softmax}\left(\frac{QK^\top}{\sqrt{d_k}}\right)V,
\end{equation}
where \(d_k\) is the dimension of the keys and queries. 

\subsection{Low-Rank Adapter (LoRA) Fine-tuning} In large transformers, these weight matrices dominate the parameter count, making them an ideal target for LoRA's efficient fine-tuning. By applying low-rank updates to these matrices, LoRA achieves significant savings in memory and computation without retraining the entire model. Consider a pre-trained weight matrix \( W_0 \in \mathbb{R}^{d_{\text{out}} \times d_{\text{in}}} \), typically representing the projection matrices in the attention mechanism. Instead of updating the entire matrix, LoRA modifies the weights by adding a low-rank perturbation:
\begin{equation}
    W = W_0 + \Delta W,
\end{equation}

where \( \Delta W \) is constrained to have \( \text{rank}(\Delta W) = r \leq \min(d_{\text{out}}, d_{\text{in}}) \). To efficiently parameterize \( \Delta W \), LoRA decomposes it as:
\begin{equation}
\Delta W = B A,    
\end{equation}

where \( B \in \mathbb{R}^{d_{\text{out}} \times r} \) and \( A \in \mathbb{R}^{r \times d_{\text{in}}} \). During fine-tuning, the original weights \( W_0 \) remain frozen, and only the parameters in \( A \) and \( B \) are optimized. In traditional full fine-tuning, the entire weight matrix is updated, requiring \( d_{\text{out}} \cdot d_{\text{in}} \) trainable parameters. In contrast, the LoRA decomposition introduces only $ r \cdot (d_{\text{in}} + d_{\text{out}})$ trainable parameters, which is more efficient when $ r \ll \min(d_{\text{out}}, d_{\text{in}})$.

\subsection{Distance Measures between Probability Distributions} 
\label{appendix:Distance Measures between Probability Distributions}
We remind the reader of the necessary definitions required in formulating the distance between datasets, when interpreted as finitely supported probability measures (distributions).
Given two probability distributions $P$ and $Q$ defined on a separable and complete (Polish) metric space $\mathcal{X}$ equipped with its Borel $\sigma$-algebra and metrized a metric $\rho:\mathcal{X}^2\to [0,\infty)$, these measures of discrepancy (or divergences) are defined as follows:

\paragraph{Wasserstein Distance (WD).} For distributions $P$ and $Q$ with cumulative distribution functions $F_P$ and $F_Q$ respectively, the $1$-Wasserstein distance is defined as:
\begin{equation}
    W_1(P, Q) \eqdef \inf_{\pi}\, \int_{(x,y)\in \mathcal{X}^2}\,\rho(x,y)\,\pi(d(x,y))
\end{equation}
where the infimum is taken over all joint probability distributions $\pi$ on $\mathcal{X}\times \mathcal{X}$ (with the product $\sigma$-algebra) whose marginals are $P$ and $Q$.

\paragraph{Kullback–Leibler Divergence (KL).} 
The KL divergence measures the relative entropy between two distributions:
\begin{equation}
    D_{KL}(P \parallel Q) 
    \eqdef 
    \begin{cases}
    \int \log_{x\in \mathcal{X}} \frac{dP}{dQ}(x) P(dx)
    & \mbox{: if } P \ll Q
    \\
    \infty & \mbox{: if } P\not\ll Q
    \end{cases}
\end{equation}
where $ P \ll Q$ denotes the absolute continuity of $P$ with respect to $Q$, and $\frac{dP}{dQ}$ denotes the Radon-Nikodym derivative, or probability density, of $P$ with respect to $Q$.

\paragraph{Jensen–Shannon Divergence (JS).} The JS divergence is a symmetrized version of KL divergence:
\begin{equation}
    D_{JS}(P \parallel Q) = \frac{1}{2} D_{KL}(P \parallel M) + \frac{1}{2} D_{KL}(Q \parallel M)
\end{equation}
where $M = \frac{1}{2}(P + Q)$.

\paragraph{Maximum Mean Discrepancy (MMD).} 
If $\mathcal{H}$ is a Reproducing Kernel Hilbert Space (RKHS) $\mathcal{H}$ of functions over $\mathcal{X}$ with reproducing kernel function $k$; then we may also define the MMD between $P$ and $Q$ by
\begin{equation}
    \text{MMD}^2(P, Q) = \mathbb{E}_{x,x' \sim P}[k(x, x')] - 
    2\mathbb{E}_{x \sim P, y \sim Q}[k(x, y)] + \mathbb{E}_{y,y' \sim Q}[k(y, y')]
.
\end{equation}
If $\mathcal{X}$ is $\mathbb{R}^d$ and $\mathcal{H}=L^2_{\gamma}(\mathbb{R}^d)$ for the standard Gaussian measure $\gamma\sim N(0,I_d)$, then $k$ is often chosen to be a Gaussian kernel, i.e., $k(x, y) = \exp(-\frac{\|x - y\|^2}{2\sigma^2})$.

\section{Proofs}
\label{s:Proofs}
We now prove the main result of our paper. We begin with the proof of our simplest result, Proposition~\ref{prop:PAC}.
\subsection{{Proof of Proposition~\ref{prop:PAC}}}
\label{s:Proofs__ss:PAC}
For any dataset $D$, note that $
        \operatorname{argmin}_{w\in \Delta_k}
            \frac1{K}\sum_{k=1}^K\,
                w_k\, \rho(D,D_k)
+
        \frac1{\alpha}\,
        \sum_{k=1}^K\, w_k\log(w_k)
$.  
Now, by \citep[Proposition 1]{WangHyndmanKratsios_2020_CJS} its unique minimizer, which we denote by $w^{\star}_D$, is given by
\[
    w^{\star}_D
=
    \frac{
        e^{-\rho(D,D_k)}
    }{
        \sum_{i=1}^K\,
        e^{-\rho(D,D_i)}
    }
=
    \operatorname{softmin}\circ \operatorname{align}(D)
.
\qquad \qedsymbol
\]

\subsection{{Simultaneous Proof of Theorem~\ref{thrm:Main} and Proposition~\ref{prop:Existence}}}
\label{s:Proofs__ssTHRMMAIN}
We now derive Proposition~\ref{prop:Existence} and Theorem~\ref{thrm:Main} within the same proof, as their derivation is most naturally undertaken together.
\paragraph{Step 1 - Existence of a Measurable Selector}
\hfill\\
We will first set up the Measurable Maximum Theorem, see e.g.~\citep[Theorem 18.19]{InfDimAnalysis_2006}.
Consider the constant correspondence 
\[
\begin{aligned}
    \varphi:
    \mathcal{D} & \twoheadrightarrow 2^{\mathbb{R}^p}\\
    D & \mapsto \operatorname{co}(\Delta \theta)
.
\end{aligned}
\]
Since $\operatorname{co}(\Delta \theta)$ is a closed, non-empty, and bounded set then the Heine-Borel theorem implies that $\operatorname{co}(\Delta \theta)$ is compact.  Whence $\varphi$ is a correspondence with non-empty, compact, and convex values.
Let $B\subseteq \mathcal{D}$ be a Borel set, then
\begin{equation}
\label{eq:BorelVerification}
    \varphi(B) \eqdef \bigcup_{D\in B}\varphi(D) = \bigcup_{D\in B} \operatorname{co}(\Delta \theta)
    =
    \operatorname{co}(\Delta \theta)
.   
\end{equation}
Since $\operatorname{co}(\Delta \theta)$ is closed it is Borel; whence, $\varphi$ is not only a weakly measurable correspondence~\citep[18.1 Definition (1)]{InfDimAnalysis_2006} but it is also a Borel measurable correspondence~\citep[18.1 Definition (3)]{InfDimAnalysis_2006}.  Thus, the correspondence $\varphi$ satisfies the requirements of~\citep[Theorem 18.19]{InfDimAnalysis_2006}.

Next, consider the objective function
\begin{equation}
\label{eq:Objective}
\begin{aligned}
    L:\mathcal{D}\times \mathbb{R}^p & \rightarrow [0,\infty)
\\
    (D,B) & \mapsto 
    \mathbb{E}_{(X,Y)\sim D}\big[
        \ell(f_{\theta+\vartheta}(X),Y)
    \big]
.
\end{aligned}
\end{equation}
We will show that $L$ is Carath\'{e}odory function by showing it is continuous.
Since $\mathcal{D}\times \mathbb{R}^p$ is a product (topological) space, then~\citep[Theorem 19.6]{MunkresTop} guarantees that $f$ is continuous if and only if each of its component functions is continuous; we show the latter.

Fix $D\in \mathcal{D}$.  Since $\mathcal{C}$ and the softmax function are locally Lipschitz, and since $\ell$ is Lipschitz, then their composition is locally Lipschitz.  Whence, for each $(x,y)\in [0,1]^{d+D}$ the map
\[
\Lambda_{x,y}:\operatorname{co}(\Delta \theta)) \ni \vartheta \mapsto \ell(\mathcal{C}_{\theta+\vartheta}(x),y))
\]
is $\lambda$-Lipschitz, for some $\lambda\ge 0$.  By Jensen's inequality we have: for each $\vartheta_1,\vartheta_2\in \operatorname{co}(\Delta \theta)$
\allowdisplaybreaks
\begin{align*}
&
\Big|
\mathbb{E}_{(X,Y)\sim \mathcal{D}}\big[
    \Lambda_{X,Y}(\vartheta_1)
\big]
-
\mathbb{E}_{(X,Y)\sim \mathcal{D}}\big[
    \Lambda_{X,Y}(\vartheta_2)
\big]
\Big|
\\
& =
\Big|
\mathbb{E}_{(X,Y)\sim \mathcal{D}}\big[
    \Lambda_{X,Y}(\vartheta_1)
-
    \Lambda_{X,Y}(\vartheta_2)
\Big|
\\
& \le 
\mathbb{E}_{(X,Y)\sim \mathcal{D}}\Big[
    \big|
        \Lambda_{X,Y}(\vartheta_1)
    -
        \Lambda_{X,Y}(\vartheta_2)
    \big|
\Big]
\\
& \le 
\lambda
\mathbb{E}_{(X,Y)\sim \mathcal{D}}\Big[
    \big\|
        \vartheta_1
    -
        \vartheta_2
    \big\|
\Big]
.
\end{align*}
Thus, $L$ is locally Lipschitz in its second argument; in particular, it is continuous in its second argument.  

Now, we show continuity in its first argument.  Fix $\vartheta\in \operatorname{co}(\Delta \theta)$.
Let $(D_n)_{n=1}^{\infty}$ be a sequence in $\mathcal{D}$ converging to some measure $D\in \mathcal{D}$.  Since $d$ metrizes the (relative) weak topology in $\mathcal{P}([0,1]^{d+D})$ relative to $\mathcal{D}$, then by Alexandrov's Portmanteau Theorem, see e.g.~\citep[Theorem 5.25]{Kallenberg}, for every continuous and bounded function $g\in C_b([0,1]^{d+D})$ we have
\begin{equation}
\label{eq:Portmanteau}
    \lim_{n\uparrow \infty}\,
    \big|
        \mathbb{E}_{(X,Y)\sim D_n}[g(X,Y)]
        -
        \mathbb{E}_{(X,Y)\sim D}[g(X,Y)]
    \big|
    =
    0
.
\end{equation}
Since $\lambda_{x,y}(\vartheta)$ is locally-Lipschitz for each $\vartheta\in \operatorname{co}(\Delta \theta)$ and $[0,1]^{d+D}$ is compact then $(x,y)\mapsto \lambda_{x,y}(\vartheta)$ is bounded (and of course continuous).  Thus, we pay pick $g$ in~\eqref{eq:Portmanteau} to be $(x,y)\mapsto \lambda_{x,y}(\vartheta)$; whence, 
\begin{equation}
\label{eq:Portmanteau2}
    \lim_{n\uparrow \infty}\,
    \big|
        \mathbb{E}_{(X,Y)\sim D_n}[\lambda_{X,Y}(\vartheta)]
        -
        \mathbb{E}_{(X,Y)\sim D}[\lambda_{X,Y}(\vartheta)]
    \big|
    =
    0
.
\end{equation}
Thus, $L$ is continuous in its first argument as well.  Therefore, $L$ is continuous, which implies that it is Carath\'{e}odory.
This completes the verification of all the conditions of the Measurable Maximum Theorem, again see~\citep[Theorem 18.19 (2)]{InfDimAnalysis_2006}, have been verified.  Whence: 1) for each $D\in \mathcal{D}$ the $\operatorname{argmin}$ set
\begin{equation*}
M(D)
\eqdef 
\Big\{
 \vartheta\in \operatorname{co}(\Delta \theta):\,
 \mathbb{E}_{(X,Y)\sim \mathcal{D}}\Big[
            \ell(f_{\theta + \Delta \theta}(X),Y)
        \Big]
=
\ell^{\star}(D)
\Big\}
\end{equation*}
is non-empty; 
where the corresponding oracle loss is given by
\[
\ell^{\star}(D)
\eqdef 
\inf_{\Delta \theta \in \operatorname{co}(\Delta \theta)}\,
        \mathbb{E}_{(X,Y)\sim \mathcal{D}}\Big[
            \ell(f_{\theta + \Delta \theta}(X),Y)
        \Big]
.
\]
This establishes Proposition~\ref{prop:Existence}.  Moreover,~\citep[Theorem 18.19 (1) and (3)]{InfDimAnalysis_2006} further imply that there exists a measurable selector
$
S:\mathcal{D} \to \operatorname{co}(\Delta \theta)
$; i.e.\ $S$ is Borel measurable for each $D\in \mathcal{D}$ the following optimal selection property holds:
\begin{equation}
\label{eq:optimality}
    S(D)\in M(D)
.
\end{equation}

\paragraph{Step 2 - Change of Domain}
\hfill\\
Next, we create a ``copy'' of $S$ in ``distance domain'' $[0,\infty)^K$. 
By the well-posedness assumption made in Assumption~\ref{ass}, the map $\operatorname{align}:D\to [0,\infty)^K$ is injective.  Thus, $\operatorname{align}$ is bijective onto its image. Since each component of $\operatorname{align}$ is given by the $1$-Lipchitz, and therefore continuous, function $\mathcal{D}:D\mapsto \rho(D,D_k)\in [0,\infty)$; then~\citep[Theorem 19.6]{MunkresTop} implies that $\operatorname{align}$ is continuous.  Consequentially, $\operatorname{align}$ is a measurable bijection onto its image $\operatorname{align}(\mathcal{D})$.
Thus, \citep[corollary 15.2]{DescriptiveSetTheory} implies that $\operatorname{align}$ has a measurable inverse 
$
\psi:\operatorname{align}(\mathcal{D}) \to \mathcal{D}
$
on its image $\operatorname{align}(\mathcal{D})$; i.e.
\begin{equation}
\label{eq:bimeasurability}
\operatorname{align}\circ \psi = 1_{\operatorname{align}(\mathcal{D})}
\mbox{ and }
\psi\circ \operatorname{align}= 1_{\mathcal{D}}
.
\end{equation}
Define the map $S^{\prime}:\operatorname{align}(\mathcal{D})\to \operatorname{co}(\Delta \theta)$ by composition with $\psi$ via
\[
    S^{\prime}\eqdef S\circ \psi
.
\]
Let $\tilde{S}$ be any measurable extension of $S^{\prime}$ to all of $\mathbb{R}^K$; e.g.
\[
    \tilde{S}\eqdef S^{\prime} I_{x\in \operatorname{align}(\mathcal{D})}
    +
    \Delta\theta_1\, I_{x\not\in \operatorname{align}(\mathcal{D})}
.
\]
By construction: for each $D\in \mathcal{D}$
\begin{equation}
\label{eq:extension_almost_complete}
\tilde{S}\circ \operatorname{align}(D)
=
S(D)
.
\end{equation}
\paragraph{Step 3 - High-Probability of Continuity}
\hfill\\
Consider the pushforward (probability) measure $\mathbb{Q}\eqdef \operatorname{align}_{\sharp}\mathbb{P}$ on $[0,\infty)^K$, supported on $\operatorname{align}(\mathcal{D})$.
Now, by Lusin's Theorem, as formulated in~\citep[Exercise 13.1.3]{Klenke}, for every $\varepsilon\in (0,1]$ there exists a compact subset $K_{\varepsilon}\subset
\operatorname{supp}(\mathbb{Q})\subseteq 
\operatorname{align}(\mathcal{D})$ such that
\begin{equation}
\label{eq:Lusin}
\mathbb{Q}(K_{\varepsilon})\ge 1-\varepsilon
\mbox{ and }
\tilde{S}|_{K_{\varepsilon}}\in C(K_{\varepsilon},\operatorname{co}(\Delta \theta))
\end{equation}
where $C(K_{\varepsilon},\operatorname{co}(\Delta \theta))$ denotes the set of continuous functions from $K_{\varepsilon}$ to $\operatorname{co}(\Delta \theta)$.  

Since $\tilde{S}|_{K_{\varepsilon}}$ is continuous \textit{and its image lies in a closed convex set} then the Dugundji-Tietze theorem, see~\citep[Theorem 4.1]{DugundjiExtension}, implies that there exists a continuous extension $S_{\varepsilon}:\mathbb{R}^K\to \operatorname{co}(\Delta \theta)$; i.e.
\begin{equation}
\label{eq:LusinExtension}
    S_{\varepsilon}(x) = \tilde{S}(x)
\end{equation}
for all $x\in K_{\varepsilon}$.

\paragraph{Step 4 - Approximation by Models of the form~\eqref{eq:Neural}}
\hfill\\
Let $W:\mathbb{R}^{K-1}\to \mathbb{R}^K$ be the affine map of~\citep[Example 13]{PaponKratsios}.  Then, by nearly identical computation to~\citep[Example 13]{PaponKratsios}, we find that the map
\begin{equation}
\label{eq:softmaxversion}
\begin{aligned}
    \mathbb{R}^K & \rightarrow \operatorname{co}(\Delta \theta)
\\
    w & \mapsto
    \operatorname{softmax}(W(w))^{\top}(L_1,\dots,L_K)
\end{aligned}
\end{equation}
also satisfies~\citep[Assumption 8]{PaponKratsios}.  Since $\operatorname{softmin}=\operatorname{softmax}(-\cdot)$; set $\tilde{W}\eqdef -W$, and note that, the result of~\eqref{eq:softmaxversion} can be re-expressed as 
\begin{equation}
\label{eq:softmin_version}
\begin{aligned}
    \mathbb{R}^K & \rightarrow \operatorname{co}(\Delta \theta)
\\
    w & \mapsto
    \operatorname{softmin}(\tilde{W}(w))^{\top}(L_1,\dots,L_K)
\end{aligned}
\end{equation}

Since $\tilde{S}$ is continuous, $K_{\varepsilon}\subset \mathbb{R}^K$ is a non-empty compact set, and $\varsigma$ is a continuous activation function satisfying the Kidger-Lyons condition, of~\cite{kidger2020universal}; namely it differentiable with non-zero derivative at at least one point in $\mathbb{R}$, then (an inconsequential mild variant) of~\citep[Theorem 37 (ii)]{PaponKratsios} implies that: for every $\delta>0$ (to be fixed retroactively) there exists an MLP $\hat{f}:\mathbb{R}^K\to \mathbb{R}^K$ with activation function $\varsigma$ such that the map
\[
    \hat{f}
    \eqdef 
    [\operatorname{softmin}\circ \mathcal{C}(\cdot)]^{\top}(L_1,\dots,L_K)
    :\mathbb{R}^K \to \operatorname{co}(\Delta \theta)
\]
satisfies the uniform approximation guarantee
\begin{equation}
\label{eq:UAT_out_way}
    \max_{x\in K_{\varepsilon}}
    \big\|
        F_{\delta}(x)
        -
        S_{\varepsilon}(x)
    \big\|
    <
        \delta
.
\end{equation}
Now, by~\eqref{eq:extension_almost_complete}, the continuous extension property in~\eqref{eq:LusinExtension}, and the approximation guarantee in~\eqref{eq:UAT_out_way} we find that
\allowdisplaybreaks
\begin{align}
& 
    \max_{D\in \psi(K_{\varepsilon})}
    \big\|
        \hat{f}
        \circ 
        \operatorname{align}(D)
        -
        S(D)
    \big\|
=
    \max_{x\in K_{\varepsilon}}
    \big\|
        \hat{f}
        -
        S_{\varepsilon}(x)
    \big\|
    <
        \delta
.
\end{align}
The continuity of $L$, defined in~\eqref{eq:Objective}, implies that $\delta>0$ may be taken to be small enough so that: for each $D\in \psi(K_{\varepsilon})$
\begin{align}
\label{eq:approximation_complete}
        \big|
            L(D,\theta + \mathcal{C}_{\delta})
        -
            L(D,\theta + S(D))
        \big|
    <
        \varepsilon
.
\end{align}
\paragraph{Step 5 - $\epsilon$-Optimality with high probability}
\hfill\\
Combining the $\varepsilon$-uniform approximation guarantee in~\eqref{eq:approximation_complete} for $\mathcal{C}_{\delta}\circ \operatorname{align}$ with the optimality guarantee for $S$ in~\eqref{eq:optimality} implies that: for each $D\in \psi(K_{\varepsilon})$
\begin{equation}
    L(D,\theta + \mathcal{C}_{\delta})
    -
    \varepsilon
    \le 
    L(D,\theta + S(D))
    =
    \ell^{\star}(D)
.
\end{equation}
Snow, since $D\mapsto L(D,\theta + \mathcal{C}_{\delta})$ is the composition of continuous functions, it is continuous and therefore measurable; whence the set 
\[
    M^{\star}_{\varepsilon}
\eqdef 
    \big\{
        L(D,\theta + \mathcal{C}_{\delta})
        -
        \varepsilon
        \le 
        \ell^{\star}(D)
    \big\}
\]
is Borel measurable and contains $\psi(K_{\varepsilon})$.  In particular, $\mathbb{P}(M^{\star}_{\varepsilon})$ is well-defined.
Finally, the lower-bound in~\eqref{eq:Lusin} yields
\[
\mathbb{P}(M^{\star}_{\varepsilon})
\ge 
\mathbb{P}(\psi(K_{\varepsilon})
\ge
\mathbb{Q}(K_{\varepsilon})
\ge 
1-\varepsilon
\]
which concludes our proof.

\section{Implementation Details}
\label{s:Implementation_Details}

In our implementation, we used the Lots-of-LoRAs HuggingFace repository~\cite{HuggingFaceLotsOfLoRAs}, which contains 502 dataset-adapter pairs for Mistral-7B-Instruct-v0.2. From these 502 English datasets, 10 are manually selected to ensure diversity of tasks spanning classification, commonsense reasoning, and question generation domains for evaluation. Additionally, 492 datasets are randomly selected from the 1616 diverse natural language processing (NLP) tasks provided by \cite{wang2022supernaturalinstructionsgeneralizationdeclarativeinstructions}. Each adapter comprises $ p = 9,437,184$ parameters, stored as 32-bit floating-point numbers (approximately 36 MB).

To further reduce the computational load of our training procedure, we implemented several critical optimizations in Step 2:

\begin{enumerate}
    \item \textbf{Symmetry exploitation}: For symmetric difference metrics (WD, JS, and MMD), we calculate only half of the possible $N \times N$ distances, reusing values obtained from calculations done for pair $(i,j)$, where $i < j$, as the $(j,i)$ pair as well.
    \item \textbf{Pre-computation of probability distributions}: For metrics requiring probability density functions (KL and JS), we pre-calculate and cache these distributions for all datasets to avoid repeating these costly computations.
    \item \textbf{Parallelization}: We also utilize multi-threading capabilities by assigning each distance calculation to a separate CPU thread, allowing these independent computations to be processed concurrently.
\end{enumerate}

Table~\ref{table:timing} reports the time elapsed at each stage of our LoRA generation pipeline, measured on a Dell XPS $15$ (Intel i$7$-$13700$H, $14$ cores, $64$ GB RAM). All steps, except for the final inference and adapter loading, were executed using the CPU only across 502 datasets. Importantly, we ran this benchmark by predicting the adapter for each of the $502$ datasets, assuming the remaining $501$ were given, to evaluate the overall performance of our pipeline.

\begin{table}[h]
\caption{Time elapsed for each step of the pipeline for all 502 datasets at once (CPU only).}
  \centering
  \small
  \begin{tabular}{@{}l r@{}}
    \toprule
    \textbf{Pipeline Step}                                 & \textbf{\hspace{2 in} Time} \\
    \midrule
    \textit{1. Dataset-Adapter pairs gathering:}             &            \\
    Downloading raw data                                     & 15 min     \\
    \addlinespace
    \textit{2. Datasets Pre-processing:}                     &            \\
    Tokenization                                             & 10 min     \\
    \addlinespace
    \textit{3. Distribution similarity calculations:}        &            \\
    \quad Wasserstein (WD)                                   & 3 hours    \\
    \quad Kullback–Leibler (KL)                              & 5 min      \\
    \quad Jensen–Shannon (JS)                                & 5 min      \\
    \quad Maximum Mean Discrepancy (MMD)                     & 1.5 hours  \\
    \addlinespace
    \textit{4. Distances Processing (Coefficients):}         &            \\
    \quad Base attentional                                   & 3 min      \\
    \quad Normalized                                         & 3 min      \\
    \quad MLP-based                                          & 45 min     \\
    \addlinespace
    \textit{5. Adapter prediction:}                          &            \\
     Calculating adapters and saving                         & 5 min      \\
    \bottomrule
  \end{tabular}
  
  \label{table:timing}
\end{table}

Excluding GPU inference and adapter loading, generating predicted adapters for $502$ datasets across $12$ methods took roughly $9$ hours. In typical use—predicting one adapter using a single variate of our LoRA generation pipeline and metric—runtime is much lower: generating one adapter from $100$ reference pairs takes $10$–$20$ minutes, depending on compute, memory, network, and dataset size.  Runtime scales roughly linearly with the number and size of reference datasets, as most steps run independently per dataset. However, full experimental runs involving pairwise comparisons (e.g., distance computations) scale quadratically with the number of datasets.

\subsection{Pipeline Steps}
\label{s:Implementation_Details__ss:Further_MethDetails}

We evaluate three pipelines for predicting LoRA adapter parameters. The \textit{Attentional} method is lightweight, using only matrix multiplications with no learned components. The \textit{Normalized} method standardizes distance values to a normal distribution to stabilize the \textit{SoftMin} stage. The \textit{Neural} method trains a small CPU-based MLP to minimize MSE between predicted and actual adapter weights and biases.  


\subsubsection{Dataset-Adapter pairs gathering} Our approach relies on pre-existing fine-tuned adapters and their corresponding datasets. We begin by gathering a set of $N$ datasets, denoted as $\{D_i\}_{i=1}^N$, where for each dataset, we also have the optimal adapters, $\{\theta_n\}_{n=1}^N$. These adapters are generated by fine-tuning the same base model, using the same adapter structure, on their respective datasets.

\subsubsection{Datasets Pre-processing}
Next, we tokenize each dataset using the base model’s tokenizer, converting inputs and outputs into integer sequences. Formally, we apply a tokenizer $T: \mathcal{S} \rightarrow \mathbb{Z}^{l \times V}$ (where $l$ being the length of the tokenized sequence and $V$ being the tokenizer's vocabulary size) to map each string to its sequence of token IDs. The resulting sequences denoted $\{T(D_i)\}_{i=1}^N$, contain all the tokenized inputs followed by the outputs for each dataset. This preprocessing step is computationally efficient and highly parallelizable. We also extract the LoRA adapter parameters (weights and biases) from each fine-tuned model, reshape them into one-dimensional vectors, and stack them into a matrix $\theta_{\text{all}} \in \mathbb{R}^{N \times p}$ ($N$ being the number of dataset-adapter pairs, and $p$ the number of parameters per adapter). Thus, each row represents the parameters of a single adapter.

\subsubsection{Distribution Distance Computation}
A key step in our pipeline is computing the dissimilarity between datasets, which are treated as probability distributions over tokenized sequences. Given tokenized datasets $\{T(D_i)\}_{i=1}^N$, we compute pairwise distances using four established measures: the Wasserstein distance, Kullback–Leibler divergence, Jensen–Shannon divergence, and Maximum Mean Discrepancy, as defined in Appendix~\ref{appendix:Distance Measures between Probability Distributions}. For each tokenized dataset $T(D_i)$, we calculate a distance vector 
$\delta_i = [\rho(T(D_i), T(D_1)), \rho(T(D_i), T(D_2)), \dots, \rho(T(D_i), T(D_N))]$, where $\rho$ is the chosen divergence metric. In practice, we mask the self-distance $\rho(T(D_i), T(D_i))$ by assigning it a large value prior to normalization. Note that here we are emphasizing the tokenization step using the $T(D_i)$ notation, whereas in the main text we often omit this.

\subsubsection{Distances Processing (With different methodologies)}
The goal here is to find how close each dataset is to the current dataset and to assign coefficients to them in such a way that these coefficients increase as the similarity increases.

\paragraph{Attentional Approach} 
In this baseline approach, we directly apply the softmin function to the distance vectors, after masking the self-corresponding entry. For each dataset $D_i$, we calculate:
\begin{equation}
w_i(j) = \text{softmin}({\delta_i(j) \mid j \in {1,2,...,N}, j \neq i})
\end{equation}
where $\delta_i(j)=\rho(T(D_i), T(D_j))$ represents the distance between the tokenized datasets.

\paragraph{Attentional Approach - With Normalization}
\label{sss:Normalized}
In this variant, we normalize each distance vector to have zero mean ($\mu = 0$) and unit variance ($\sigma = 1$), effectively applying z-score standardization. This transformation is equivalent to applying a softmin with an adaptive temperature $\tau_i = \sigma_i$ (its own standard deviation). When $\sigma_i$ is small, the temperature is low, leading to sharper, more peaked (i.e., sparse) coefficient distributions. Conversely, larger $\sigma_i$ results in flatter distributions. Empirically, we observe that most $\sigma_i$ values are small after masking the self-distance, which leads to sparser weights—and, interestingly, improved performance.

\paragraph{Neural Approach}
The third pipeline, justified by Theorem~\ref{thrm:Main}, uses a small MLP to map distance values to adapter weights. It minimizes the MSE between predicted and actual adapter parameters (weights and biases). The MLP used here has three fully connected layers, with the first two followed by layer normalization and ReLU activations.
\begin{align}
h &= \text{ReLU}(\text{Layer Normalization}(W_1 x + b_1)), \quad h \in \mathbb{R}^{4000} \\
\hat{h} &= \text{ReLU}(\text{Layer Normalization}(W_2 h + b_2)), \quad \hat{h} \in \mathbb{R}^{4000} \\
\hat{y} &= W_3 \hat{h} + b_3, \quad \hat{y} \in \mathbb{R}^{1} 
\end{align}
where $x \in \mathbb{R}$ is a single distance value (scalar), $W _1 \in \mathbb{R}^{4000 \times 1}$,  $W_2 \in \mathbb{R}^{4000 \times 4000}$, and $W_3 \in \mathbb{R}^{1 \times 4000}$ are weight matrices, and $b_1 \in \mathbb{R}^{4000}$, $b_2 \in \mathbb{R}^{4000}$, and $b_3 \in \mathbb{R}^{1}$ are bias terms.
We apply the MLP to transform all distance values:
\begin{equation}
w_i(j) = \text{softmin}({\text{MLP}(\delta_i(j)) \mid j \in {1,2,...,N}, j \neq i})
\end{equation}

\subsubsection{Adapter Prediction}
We make our prediction with a straightforward linear combination of existing adapters, weighted by the processed distances:
\begin{equation}
\hat\theta_i = \sum^{N}_{j=1, j\neq i} w_i(j) \theta_j.
\end{equation}

This formulation effectively answers the key question: ``Based on the distances between a new dataset and each of the datasets with known adapters, what proportion of information should the new adapters inherit from each of the fine-tuned (reference) adapters?'' The processed distances serve as coefficients determining the knowledge transfer from each source adapter.

In our study, to make the predictions for all datasets more efficient, we construct a weight (coefficient) matrix $W \in \mathbb{R}^{N \times N}$ where row $i$ contains the processed distances $w_i$, allowing us to compute all predictions simultaneously by leveraging hardware acceleration and vectorization.

\paragraph{Deployment and Inference} Predicted adapters match the size of flattened fine-tuned adapters and can be reshaped to their original structure, ensuring full compatibility with existing LoRA inference pipelines. Once generated, they can be directly loaded for downstream use.

\section{Further Experimental Evaluation}
This appendix presents a detailed account of our experimental observations.

\subsection{Exact Match Evaluation}\label{appendix: exact match}

In addition to Rouge-L, we evaluate our LoRA generation pipelines using the Exact Match (EM) metric, which measures the fraction of test samples for which the model's output exactly matches the expected string. This is a particularly meaningful complement for classification-style tasks common in our dataset corpus, where outputs are short, well-defined, and often categorical. Without any fine-tuning, the Mistral model achieved a score of $0.016 \pm 0.069$. Ideally, if the user had access to GPUs, the GPU fine-tuned models would achieve an average exact match score of $0.654 \pm 0.351$). As shown in Table~\ref{table:ExactMatchPerformance}, each of our pipelines performs substantially better than the base foundation model, but as expected, it does not achieve the same predictive power as LLMs with fine-tuned LoRAs. Additionally, note that we observe a strong correlation between Rouge-L and EM scores across all methods and distance metrics. Both evaluation scores consistently rank the Normalized approach with JS as the top-performing configuration. While Rouge-L captures partial overlap between generated and reference sequences, EM provides a stricter binary signal of correctness. Despite this difference in granularity, the relative performance of the Attentional, Normalized, and Neural approaches remains consistent suggesting that improvements in soft sequence similarity are accompanied by gains in exact prediction accuracy.

\begin{table}[hbpt!]
\caption{Exact match performance of our lightweight LoRA prediction pipelines. }
\centering
\small
\setlength{\tabcolsep}{10pt}
\renewcommand{\arraystretch}{1.2}
\adjustbox{width=0.6\linewidth}{
\begin{tabular}{lcccc}
\toprule
Approach & WD & KL & JS & MMD \\
\midrule
Attentional     & $0.288\,{\pm}\,0.297$ & $0.344\,{\pm}\,0.302$ & $0.328\,{\pm}\,0.296$ & $0.327\,{\pm}\,0.295$ \\
Normalized      & $0.338\,{\pm}\,0.296$ & $0.330\,{\pm}\,0.297$ & $0.373\,{\pm}\,0.314$ & $0.340\,{\pm}\,0.298$ \\
Neural          & $0.338\,{\pm}\,0.294$ & $0.323\,{\pm}\,0.295$ & $0.325\,{\pm}\,0.296$ & $0.337\,{\pm}\,0.298$ \\
\bottomrule
\end{tabular}
}
\label{table:ExactMatchPerformance}
\end{table}
\vspace{-0.5em}


\subsection{Coefficient Distribution Analysis}

Figures~\ref{fig:coeff_base},~\ref{fig:coeff_normalized}, and~\ref{fig:coeff_mlp} below show the LoRA matrices produced by each approach across each of our datasets.  In each visualization, both the horizontal and vertical axes list the dataset, and each of the $(i,j)^{th}$ pixel darkness indicators how much of the pre-trained LoRA from dataset $i$ is used to predict the LoRA for dataset $j$.  Darker pixels indicate lower coefficients, while brighter ones indicate higher weights assigned to a source adapter for each target dataset. Both axes correspond to dataset indices. Interestingly, the Normalized approach exhibits extreme sparsity: most weights are near zero, and each prediction is dominated by one or two reference adapters, as evidenced by the presence of isolated bright pixels in a largely dark matrix. In contrast, the Neural (MLP) and Attentional methods display greater dispersion in the coefficients, with many rows exhibiting moderate weights across a broader range of source adapters. This reflects a more distributed and nuanced reuse of prior adapters. Given that the Normalized approach exhibits slightly better performance in practice, this visualization may suggest that sparsity is important, but further investigation in follow-up work is encouraged.


\begin{figure}[H]
    \centering
    \begin{subfigure}{.32\textwidth}
        \centering
        \includegraphics[width=.95\linewidth, trim=0 0 0 0, clip]{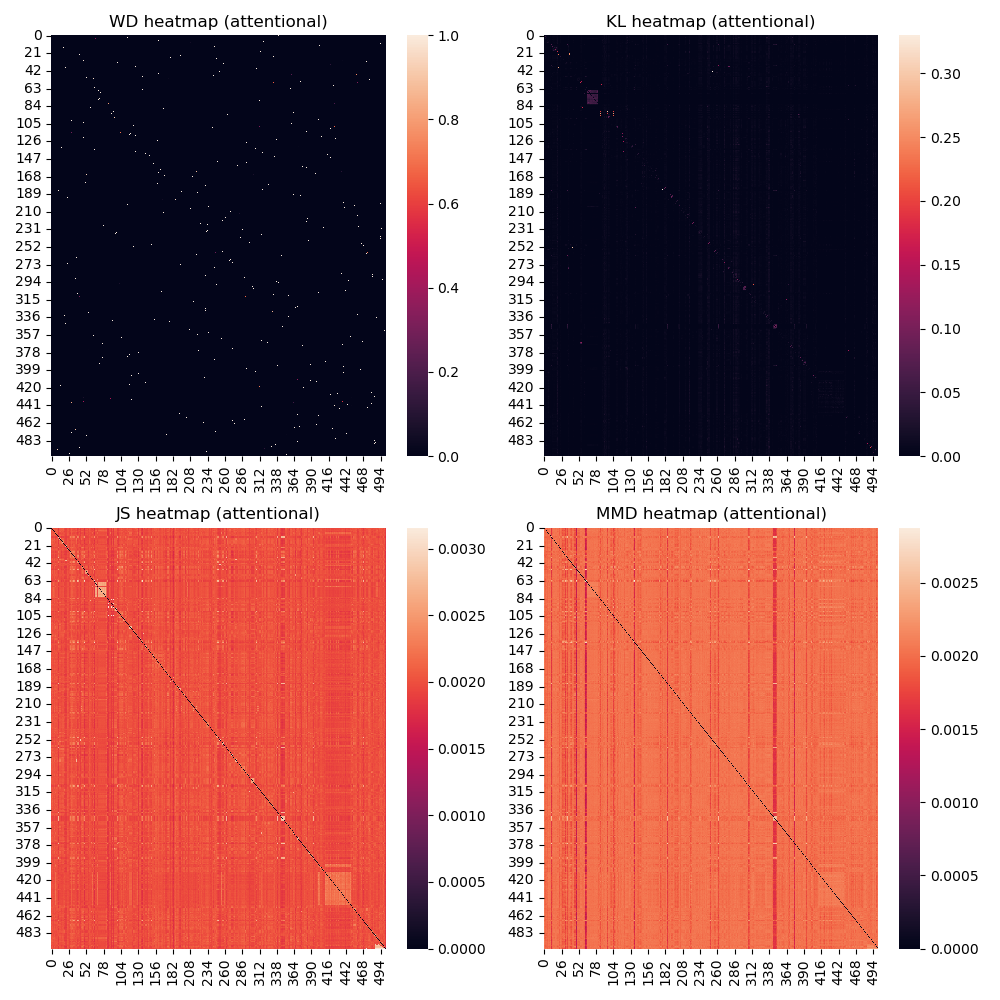}
        \subcaption{Attentional}
        \label{fig:coeff_base}
    \end{subfigure}%
    \hfill
    \begin{subfigure}{.32\textwidth}
        \centering
        \includegraphics[width=.95\linewidth, trim=0 0 0 0, clip]{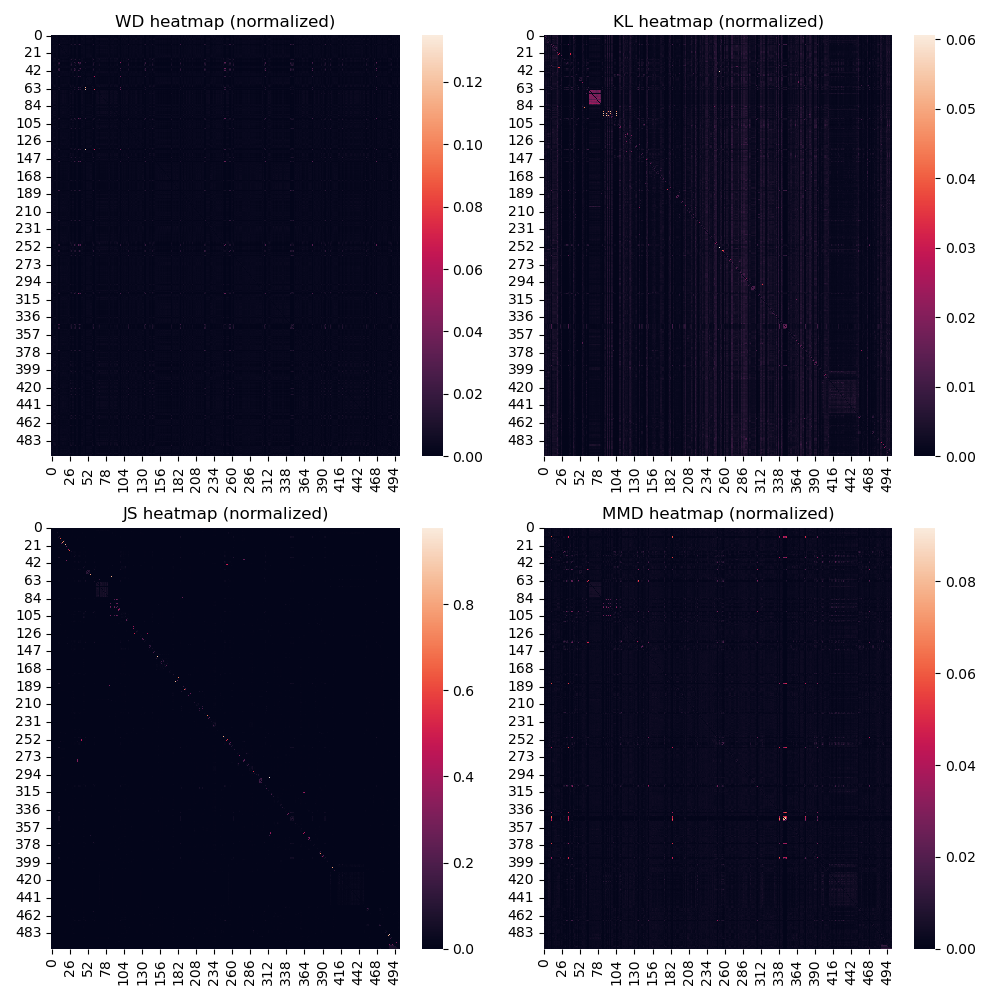}
        \subcaption{Normalized}
        \label{fig:coeff_normalized}
    \end{subfigure}%
    \hfill
    \begin{subfigure}{.32\textwidth}
        \centering
        \includegraphics[width=.95\linewidth, trim=0 0 0 0, clip]{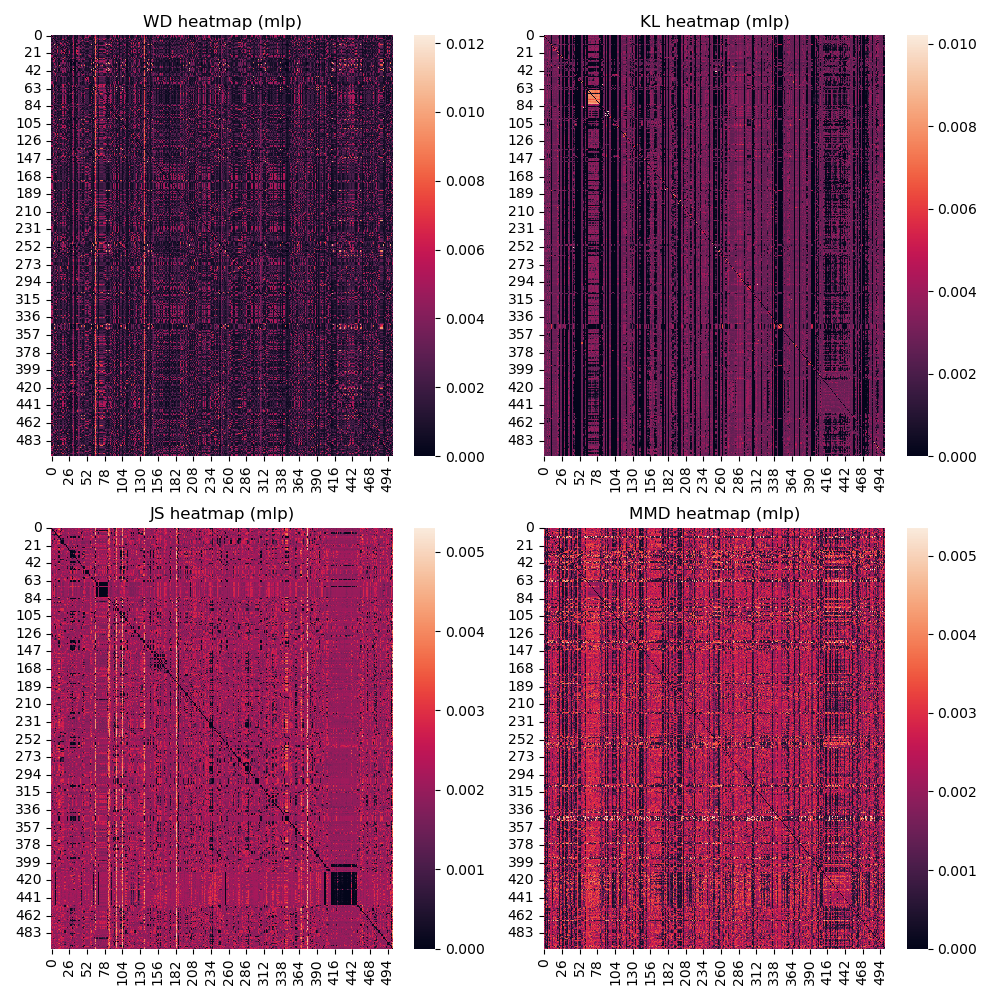}
        \subcaption{Neural (MLP)}
        \label{fig:coeff_mlp}
    \end{subfigure}
    
    \caption{Coefficient distributions for each approach. Each image: Top Left = WD, Top Right = KL, Bottom Left = JS, Bottom Right = MMD.}
    \label{fig:coeff_all}
\end{figure}

\subsection{Sample Outputs Generated by Predicted Models}

Below, we have included four randomly selected inputs, expected outputs, and the generated texts by a model with predicted adapters using our best configurations: the JS-divergence-based Normalized approach.

\begin{small}
\begin{tcolorbox}[enhanced,breakable,userstyle,origprompt]
\textbf{Input Prompt [Task ID 671]}: \\
"Definition: In this task, you're given an ambiguous question (which can be answered in more than one way). Your task is to provide one question which clarifies the input question and it has one unique answer, and also provide an answer to the generated question. Generated question and answer should be separated with a new line.\\ \\Positive Example 1 -\\Input: Lucifer how many episodes are in season 3?\\Output: How many episodes of season 3 of Lucifer were there, including bonus episodes? \\ 26\\ \\Positive Example 2 -\\Input: What is the tallest ride at six flags over texas?\\Output: What is the tallest roller coaster at six flags over texas 2001? \\ The Titan\\ \\Negative Example 1 -\\Input: Lucifer how many episodes are in season 3?\\Output: How many episodes of season 3 of Lucifer were there, including bonus episodes? 26\\ \\Negative Example 2 -\\Input: What is the tallest ride at six flags over texas?\\Output: What is the tallest roller coaster at six flags over Texas 2001?\\ \\Now complete the following example -\\Input: When was season 7 of game of thrones released?\\Output: \\ \\"
\end{tcolorbox}

\begin{tcolorbox}[userstyle,botstyle]
\textbf{Expected Answer}: \\
"When was season 7, episode 1 of game of thrones released through simulcast? \\ July 16, 2017"
\end{tcolorbox}

\begin{tcolorbox}[userstyle,rephrasestyle]
\textbf{Model Output}: \\
"When was season 7 of game of thrones released in the US?\\July 16, 2017
\end{tcolorbox}

\vspace{5mm}

\begin{tcolorbox}[enhanced,breakable,userstyle,origprompt]
\textbf{Input Prompt [Task ID 431]}: \\
"Definition: In this task you are given a sentence. You must judge whether the object of the main clause is singular(like: apple) or plural(like: apartments). Label the instances as "Singular" or "Plural" based on your judgment.\\ \\Positive Example 1 -\\Input: They both let out relieved breaths at the doctor\'s words.\\Output: Plural\\ \\Positive Example 2 -\\Input: He signaled the bartender and ordered one more drink.\\Output: Singular\\ \\Negative Example 1 -\\Input: I collect my painting supplies.\\Output: Singular\\ \\Negative Example 2 -\\Input: I saw the ramp leading back toward the surface.\\Output: Plural\\ \\Now complete the following example -\\Input: For some unknown reason, that second kiss wreaked more emotional havoc than the first had.\\Output: \\ \\"
\end{tcolorbox}

\begin{tcolorbox}[userstyle,botstyle]
\textbf{Expected Answer}: \\
"Singular"
\end{tcolorbox}

\begin{tcolorbox}[userstyle,rephrasestyle]
\textbf{Model Output}: \\
"Singular"
\end{tcolorbox}

\vspace{5mm}

\begin{tcolorbox}[enhanced,breakable,userstyle,origprompt]
\textbf{Input Prompt [Task ID 322]}: \\
"Definition: In this task, you are given a public comment from online platforms. You are expected to classify the comment into two classes: threat and non-threat. Threat is a statement that someone will be hurt or harmed, especially if the person does not do something in particular.\\ \\Positive Example 1 -\\Input: Comment: Enough already.  Cut the head off this snake.\\Output: Threat\\ \\Positive Example 2 -\\Input: Comment: Disgusting comment, I pity anyone who has to live with you.\\Output: Non-threat\\ \\Negative Example 1 -\\Input: Comment: Bravo!  Bravo!\\ \\Encore!! Encore!! Encore!!\\Output: Threat\\ \\Negative Example 2 -\\Input: Comment: ... after shooting her with her own gun.\\Output: Non-threat\\ \\Now complete the following example -\\Input: Comment: It's a good thing that right wing illiterates in farms don't read the newspaper.\\Output: \\ \\"
\end{tcolorbox}

\begin{tcolorbox}[userstyle,botstyle]
\textbf{Expected Answer}: \\
"Non-threat"
\end{tcolorbox}

\begin{tcolorbox}[userstyle,rephrasestyle]
\textbf{Model Output}: \\
"Non-threat"
\end{tcolorbox}

\vspace{5mm}

\begin{tcolorbox}[enhanced,breakable,userstyle,origprompt]
\textbf{Input Prompt [Task ID 1398]}: \\
"Definition: Given a fact, create a question that can be answered using the fact. Construct the question such that it is unambiguous, has a unique answer and the answer can be given using the fact.\\ \\Positive Example 1 -\\Input: Fact: deep sea animals live deep in the ocean\\Output: Frilled sharks and angler fish live far beneath the surface of the ocean, which is why they are known as?\\ \\Positive Example 2 -\\Input: Fact: as an object moves , the kinetic energy of that object will increase\\Output: An example of lots of kinetic energy would be?\\ \\Negative Example 1 -\\Input: Fact: water is often brackish in an estuary\\Output: What is the sun made of?\\ \\Negative Example 2 -\\Input: Fact: if a liquid disappears then that liquid probably evaporated\\Output: What happens is water is mopped up?\\ \\Now complete the following example -\\Input: Fact: as the use of a crop increases , the amount of crops planted will increase\\Output: \\ \\"
\end{tcolorbox}

\begin{tcolorbox}[userstyle,botstyle]
\textbf{Expected Answer}: \\
"When the demand for corn rises?"
\end{tcolorbox}

\begin{tcolorbox}[userstyle,rephrasestyle]
\textbf{Model Output}: \\
"Which crop is most likely to be planted in large quantities due to its high demand?"
\end{tcolorbox}
\end{small}







\end{document}